\documentclass[12pt]{arxiv2022} %

\usepackage[utf8]{inputenc} %
\usepackage[T1]{fontenc}    %
\usepackage{hyperref}       %
\usepackage{url}            %
\usepackage{booktabs}       %
\usepackage{amsfonts}       %
\usepackage{nicefrac}       %
\usepackage{microtype}      %
\usepackage{xcolor}         %

\newcommand{\dvc}{d_{\mathrm{vc}}}
\title[Information-Theoretic Bayes Risk Lower Bounds for Realizable Models]{Information-Theoretic Bayes Risk Lower Bounds for\\ Realizable Models}
\usepackage{times}

\altauthor{%
 \Name{Matthew Nokleby} \Email{matthew@lily.ai}\\
 \addr Lily AI, Mountain View, CA
 \AND
 \Name{Ahmad Beirami} \Email{beirami@fb.com}\\
 \addr Facebook AI, Menlo Park, CA
}

\begin{document}

\maketitle

\begin{abstract}
We derive information-theoretic lower bounds on the Bayes risk and generalization error of realizable machine learning models. In particular, we employ an analysis in which the rate-distortion function of the model parameters bounds the required mutual information between the training samples and the model parameters in order to learn a model up to a Bayes risk constraint. For realizable models, we show that both the rate distortion function and mutual information admit expressions that are convenient for analysis. For models that are (roughly) lower Lipschitz in their parameters, we bound the rate distortion function from below, whereas for VC classes, the mutual information is bounded above by $\dvc\log(n)$. When these conditions match, the Bayes risk with respect to the zero-one loss scales no faster than $\Omega(\dvc/n)$, which matches known outer bounds and minimax lower bounds up to logarithmic factors. We also consider the impact of label noise, providing lower bounds when training and/or test samples are corrupted.

\end{abstract}


\section{Introduction}
A research challenge over the past decade has been to give a theoretical account of the generalization performance of high-dimensional machine learning models. In particular, {\em realizable} models, which fit perfectly to the training set, generalize better to unseen data than predicted by classical results based on Vapnik-Chervonenkis (VC) dimension or the Rademacher complexity \citep{ZhangBHRV17}, challenging textbook notions of overfitting and the bias-variance tradeoff \citep{dar2021farewell}. In response, researchers have developed new theoretical machinery for understanding and predicting generalization error, including tighter Rademacher \citep{BartlettFT17} and PAC Bayes \citep{NeyshaburBMS18,DziugaiteHGA21} analysis, explanations of the double descent phenomenon \citep{nakkiran2019deep}, and conditions for ``benign'' overfitting \citep{BartlettLL20}.

This paper follows two recent lines of work in understanding the sample complexity of high-dimensonal models. The first is the use of information-theoretic measures to understand sample complexity. \citet{pmlr-v51-russo16} and \citet{XuR17} bound the generalization error in terms of the mutual information between the training samples and the output of a learning algorithm; small mutual information is sufficient to ensure generalization. \citet{SteinkeZ20} derive tighter generalization upper bounds by introducing a ``ghost sample'' and analyzing the {\em conditional} mutual information. Along somewhat different lines, \citet{XuR20} upper bound the minimum excess risk (MER) of Bayes models, giving average-case sample complexity bound in terms of mutual information. \citet{NoklebyB16} and \citet{KolahiMKB21} give complementary lower bounds on Bayes risk using techniques from rate-distortion theory.

Second is an emphasis on the ``fast'' convergence rates achieved when models are (nearly) realizable. \citet{mammen1995asymptotical} and \citet{tsybakov2004optimal} show that, neglecting complexity and log terms, generalization error decays as $O(1/n)$ for VC classes satisfying a margin property that roughly corresponds to the data being nearly separable. These bounds are tight in the minimax sense. More recently, \citet{SteinkeZ20} and \citet{grunwald2021pac} show that $O(1/n)$ rates are achievable for a wide variety of learners, including VC classes and learners that ensure differential privacy, by extending the recent line of work on information-theoretic arguments for sample complexity.

In this paper, we seek to understand in what senses these ``fast'' convergence rates are optimal. In particular, we derive lower bounds on the excess Bayes risk of realizable models via a rate-distortion analysis similar to \cite{NoklebyB16,KolahiMKB21}. These results bound the {\em average-case} performance from below. We provide conditions in which these lower bounds are tight, in which case they generalize the minimax tightness discussed above to tightness on the average over a model family.

\subsection{Summary of Contributions}
We derive lower bounds on the excess Bayes risk for realizable models. These bounds follow from a {\em rate-distortion} analysis of a Bayesian model family, in which the rate-distortion function $R(D)$ specifies the number of bits needed to learn a model, and the mutual information $I(Z^n;W)$ specifies the number of bits provided by the training set about the model. In Sections \ref{sect:bayes.learning} and \ref{sect:rd.bounds}, we formalize the problem setting and review the rate-distortion framework.
In Section \ref{sec:bounds-on-rate-distortion}, we present conditions under which it is possible to bound $R(D)$ for realizable models. Whenever the classifier is sufficiently sensitive to perturbations in the model parameters, a Shannon lower bound on $R(D)$ holds, and the rate-distortion function is proportional to the number of model parameters $d_w$.
In Section \ref{sec:bounds-on-MI}, we present bounds on the mutual information $I(Z^n;W)$. For realizable models we derive a novel expression for $I(Z^n;W)$, and we further show that for VC classes the mutual information is bounded by $O(\dvc\log(n/\dvc))$. In this case, we obtain excess risk bounds that scale as $\Omega((\dvc/n)^{\dvc/d_w})$; in the case where $\dvc = d_w$ we recover the fast rates discussed above. We demonstrate the bounds on two simple cases, and show conditions in which the bounds are not tight.
In Section \ref{sect:mer}, we relate our results to the recent results of \citet{XuR20} and \citet{KolahiMKB21} on the {\em minimum} excess risk, deriving a new outer bound on the MER for realizable models that matches the lower bounds derived in the previous section.
In Section \ref{sect:margins}, we consider the impact of label noise on our bounds. Under mild conditions, we show that label noise on the {\em training} set has a negligible order-wise impact on the excess risk. We also consider noisy training and test distributions that satisfy the Tsybakov-Mammen margin condition mentioned above, deriving a margin-dependent lower bound on the excess Bayes risk.
Finally, in Appendix \ref{app:smooth.models} we include for completeness analogous results for smooth, non-realizable models.

\section{Problem Formulation}\label{sect:bayes.learning}
\vspace{-.1in}
\paragraph{Notation.}
We use uppercase letters to refer to random variables and lowercase letters to refer to their realizations, i.e. $X = x$. We use $P_X$ to indicate distributions chosen by ``nature'' and considered fixed, and we use $Q_X$ to indicate distributions designed by the learner. For continuous or discrete distributions $P_X$, we let $p_X(x)$ indicate the density or mass function. Let $\mathbb{E}_X[X]$ denote the expectation of random variable $X$ (with respect to $P_X$) and $\mathbb{P}_X[E]$ denote the probability of the event $E$. When possible we suppress subscripts. Finally, we use $\mathbb{I}[E]$ for the Iverson bracket, i.e. $\mathbb{I}[E] = 1$ iff $E$ is true and 0 otherwise.

\vspace{-.05in}
\subsection{Bayesian Statistical Learning}
\vspace{-.05in}
We consider statistical learning in a parametric Bayesian setting. The task is to learn a (perhaps random) function that predicts targets $Y \in \mathcal{Y}$ from input samples $X \in \mathcal{X} \subset \mathbb{R}^d$. Sample pairs $Z := (X,Y) \in \mathcal{Z} := \mathcal{X} \times \mathcal{Y}$ are drawn according to a distribution that is a member of a parametric family, i.e. $(X,Y) \sim P_{X,Y | W=w}$, where $w \in \mathcal{W} \subset \mathbb{R}^{d_w}$ is the vector of model parameters. Different from the standard statistical learning scenario, we treat the model parameters $W$ as a random variable with prior distribution $P_W$, and we suppose that both the model family and the prior $P_W$ are known to the learner. Instead of trying to determine which of a set of classifiers best predicts the data, we attempt to determine which of a set of data distributions---indexed by $W$---best fits the data.

Define the training set $Z^n := \{Z_1, \cdots, Z_n\}$ of $n$ samples drawn i.i.d. from $P_{Z | W} := P_{X,Y|W}$. A {\em parametric learning rule} is a distribution $Q_{\widehat{W} | Z^n }$ that generates estimates $\widehat{W}$ of the model parameters from the training data. Two ``reasonable'' choices are $Q_{\widehat{W} | Z^n } = P_{W | Z^n}$, in which estimates are drawn according to the posterior, and $Q_{\widehat{W} | Z^n } = \delta_{\widehat{W}}(\arg\max_w p(w|Z^n))$, the Dirac measure on the maximum a posteriori estimate of $W$.
However, we do not make specific assumptions about $Q_{\widehat{W} | Z^n }$ and instead lower bound performance over all possible learning rules.

We assess the performance of a learning rule $Q_{\widehat{W} | Z^n}$ via a loss function $\ell: \mathcal{Z} \times \mathcal{W} \to \mathbb{R}$ that scores the estimate $\widehat{W}$ against a test sample $Z$ drawn from $P_{Z | W}$. In this paper we are primarily interested in the 0-1 loss:
\begin{equation}
    \ell_{\text{0-1}}(z;\widehat{w}) = \mathbb{I}\left[\arg\max_{y^\prime} p(y^\prime | x, \widehat{w}) \neq y\right],
\end{equation}
but many of our results can be generalized to other loss functions, including the cross-entropy loss:
\begin{equation}
    \ell_{\text{ce}}(z;\widehat{w}) = -\log p(y|x, \widehat{w}).
\end{equation}

We are interested in the loss averaged over $P_W$, $P_{Z|W}$, and $Q_{\widehat{W}|Z^n}$. We suppose the test and training samples are taken to be conditionally independent given $W$, i.e. 
$
    P_{Z, Z^n, W} = P_{Z|W} \otimes \left( \prod_{i=1}^n P_{Z_i | W} \right) \otimes P_W.
$

\begin{definition}[Excess Loss]
For every realization pair $w, v \in \mathcal{W}$, the {\em excess loss} is the cost of choosing $\widehat{W} = v$ when $W=w$ is the true model, averaged over $P_{Z|W=w}$:
\begin{equation}
    D(w || v) := \mathbb{E}_{P_{Z|W=w}} [\ell(Z ; v) - \ell(Z;w)].
\end{equation}
\end{definition}
We suppose that $D(w || v)$ is a {\em quasi-divergence}, i.e. $D(w || v) \geq 0$, with equality if $v = w$. We do {\em not} assume $D(w || v) = 0$ only if $w = v$. This choice accounts for overparameterized models, for which there may be $w \neq v$ such that $p(z|w) = p(z|v)$.

Taking a further expectation over $Q_{\widehat{W}|Z^n}P_W$ gives the {\em Bayes risk}. 
\begin{definition}[Bayes Risk]
    The {\em Bayes risk} of a parametric learning rule  $Q_{\widehat{W} | Z^n}$ given $n$ training samples, $Z^n$, is its expected loss over the joint distribution $P_{Z, Z^n, W, \widehat{W}} = Q_{\widehat{W}|Z^n}P_{Z^n|W}P_{Z|W}P_{W}$:
    \begin{equation}
        L_n(Q_{\widehat{W}|Z^n}) := \mathbb{E}_{Z,Z^n,W,\widehat{W}}\left[ \ell(Z;\widehat{W})\right].
    \end{equation}
\end{definition}
We compare the Bayes risk of the learning rule $Q_{\widehat{W}|Z^n}$ to the expected loss when $W$ is known.
\begin{definition}[Excess Bayes Risk]
    The {\em excess Bayes risk} of learning rule $Q_{\widehat{W}|Z^n}$ given $n$ training samples is the gap between the Bayes risk and the ``oracle'' Bayes risk:
    \begin{equation}
        E_n(Q_{\widehat{W}|Z^n}) := \mathbb{E}_{W, Z^n, \widehat{W}} [ D(W\| \widehat{W})] =  L_n(Q_{\widehat{W}|Z^n}) - L^*,
    \end{equation}
    where
    \begin{equation}
        L^* := \mathbb{E}_{Z,W}\left[ \ell(Z;W)\right].
    \end{equation}
\end{definition}
Because the excess loss is a quasi-divergence, the excess Bayes risk is always non-negative.

By defining the loss function in terms of parameter choices $\widehat{w} \in \mathcal{W}$, we implicitly restrict our attention to (perhaps randomized) plug-in classifiers that choose $\widehat{W}$ from $Z^n$, it classifies test samples $X$ according to $P_{Y|X,W=\widehat{W}}$. This is not optimal in general; typically the Bayes-optimum classifier is based on the posterior $P_{Y|X,Z^n} = \mathbb{E}_{W} P_{Y|X,W} \otimes P_{W | Z^n}$. Nevertheless, modern machine learning practice tends toward the learning of a single model parameter, and restricting attention to the plug-in case allows us to derive stronger results, so we focus on this case. In Section \ref{sect:mer} we also discuss in detail the performance gap in making the restriction to parametric learning rules.

\subsection{Realizable models}
Our primary focus is realizable models, i.e. model families $P_{Z|W}$ where the target $Y$ is a deterministic function of a test point $X$ given $W$.
\begin{definition}[Realizable model]
    We say that a model family is {\em realizable} if $p(y|x,w) \in \{0,1\}$ for every $x \in \mathcal{X}, w \in \mathcal{W}$.
\label{def:realizable}
\end{definition}
\citet{SteinkeZ20} consider realizable models from a frequentist standpoint: the concept class is taken to be rich enough that it can fit the training set with negligible error. We make the further assumption that the underlying data distribution is separable; this is a strengthened version of the Tsybakov-Mammen condition \citep{tsybakov2004optimal}, and ensures that the training set can be fit perfectly as $n \to \infty$. In Section \ref{sect:margins} we relax this assumption somewhat and consider the impact of label noise on our bounds.

\subsection{Excess Bayes Risk vs. Generalization Error}
The excess Bayes risk measures a performance gap that is both similar to and separate from the generalization gap. Whereas excess risk indicates the gap between the test performance of a learned model and the performance of the best-possible model, the generalization error indicates the gap between a learned model's performance on the training data and on unseen data. 
\begin{definition}[Expected generalization gap]
The {\em expected generalization gap} of learning rule $Q_{\widehat{W}|z^n}$ given $n$ samples from the model is defined as
\begin{equation}
    G_n(Q_{\widehat{W}|Z^n}) = \mathbb{E}_{Z, Z^n, W,\widehat{W}} \left[\ell(Z; \widehat{W}) - \frac{1}{n} \sum_{i \in [n]}\ell(Z_i; \widehat{W})\right].
\end{equation}
\end{definition}

As we will see below, $E_n(Q_{\widehat{W}|z^n})$ and $G_n(Q_{\widehat{W}|z^n})$ often have the same scaling laws. However, for realizable models, these gaps can be identical. This is because the optimal model has zero test error, and it is always possible to obtain zero training error.
\begin{lemma}\label{lem:risk.and.generalization}
    For realizable models and a learning rule $Q_{\widehat{W}|Z^n}$ that interpolates any training set,
\begin{equation}
     E_n(Q_{\widehat{W}|Z^n}) =  G_n(Q_{\widehat{W}|Z^n}) =  L_n(Q_{\widehat{W}|Z^n}) .
\end{equation}    
\end{lemma}
Indeed, this equivalence holds in a stronger pointwise sense for any training dataset, $z^n$. 
Lemma~\ref{lem:risk.and.generalization} shows that a lower bound on the excess Bayes risk for a realizable model is also a lower bound on the generalization error (averaged over a prior) of that model family, treated as a concept class from a frequentist perspective. 

\vspace{-.08in}
\section{Information-theoretic Risk Bounds}\label{sect:rd.bounds}
\vspace{-.08in}

In the following, we derive lower bounds on the excess Bayes risk via analysis of the {\em rate-distortion function} $R(D)$ of the model parameters. The rate-distortion function \citep[Ch. 10]{CoverT06} of a random variable is used ordinarily to characterize the fundamental limits of lossy compression with respect to a distortion measure, e.g. how many bits does it take to represent $W$ up to a certain average squared error? Here, we compute the rate-distortion function $R(D)$ using the excess loss as the distortion measure, and we show that it characterizes the number of bits the {\em training set} $Z^n$ must provide about $W$ in order for the Bayes risk to be bounded by $D$.

By the data-processing inequality, we show that the number of bits conveyed by the training set $Z^n$ about $W$ is bounded by their mutual information $I(Z^n;W)$. This leads to the inequality $R(D) \leq I(Z^n;W)$ for any parametric learning algorithm: There exists a rule $Q_{\widehat{W} | Z^n}$ with excess Bayes risk $D$ only if $R(D) \leq I(Z^n;W)$. This idea is described in the context of estimation theory in \citep[Ch. 11-12]{wu2017lecture}.  \citet{NoklebyB16} and \citet{KolahiMKB21} use similar ideas to study the sample complexity of machine learning models.

We formalize this notion in the following theorem, which is a straightforward consequence of the data processing inequality and which appears in slightly different form in \citet{NoklebyB16} and \citet{KolahiMKB21}.

\begin{theorem}[\citet{NoklebyB16}, \citet{KolahiMKB21}]\label{thm:dpi}
    Let
    \begin{equation}\label{eqn:rd.definition}
        R(D) := \inf_{Q_{\widehat{W}|W}} I(W;\widehat{W}),  \quad \mathrm{s.t.} \quad \mathbb{E}_{Z,W, \widehat{W}} [\ell(Z;\widehat{W})] - L^* \leq D
    \end{equation} 
    be the {\em rate-distortion function} of $W$ with respect to loss $\ell(z;w)$, 
    where the infimum is over all random encoders $Q_{\widehat{W}|W}$. Then, there exists a learning rule $Q_{\widehat{W}|Z^n}$ with excess Bayes risk $E_n(Q_{\widehat{W}|Z^n}) \leq D$ only if
    \begin{equation}\label{eqn:dpi}
        I(Z^n; W) \geq R(D).
    \end{equation}
\end{theorem}
 
Theorem \ref{thm:dpi} provides a recipe for deriving lower bounds on the excess Bayes risk. If one can bound or derive tight estimates of $R(D)$ and $I(Z^n; W)$, one can lower bound the Bayes risk for sample size $n$ by finding the largest value of $D$ such that $R(D) \leq I(Z^n; W)$. We emphasize that the bounds derived from these methods hold in the {\em average case} over the prior $P_W$ and the training set $P_{Z^n}$. Further, these bounds are algorithm independent, holding for any parametric learning rule $Q_{\widehat{W}|Z^n}$. In the next two sections we will present techniques for bounding $R(D)$ and $I(Z^n;W)$.

We emphasize a distinction between the mutual information $I(Z^n;W)$ considered here and the mutual information between training samples and learning rule output considered in \citet{XuR17,SteinkeZ20}. In the latter case, the mutual information is a measure of sensitivity to the training set, and higher mutual information implies worse generalization guarantees. Here, on the other hand, high mutual information implies that the training set is informative about the {\em true} model parameters and implies good generalization performance.

\section{Bounding $R(D)$}
\label{sec:bounds-on-rate-distortion}
Computing the rate-distortion function in closed-form is challenging beyond classic textbook examples, and is especially difficult for realistic machine learning models. In this section we introduce tools for computing closed-form estimates of $R(D)$.

As mentioned above, we are interested in realizable models and use the 0-1 loss to measure the performance of a classifier. In this setting, model performance is determined entirely by the {\em discrepancy} between the learned classifier and the ground-truth classification rule. To bound $R(D)$, we introduce a notion of {\em sensitivity} of the classification regions with respect to perturbations in $W$; this notion is similar to Lipschitz continuity from below. When $P_{Y|X,W}$ is parameterized such that the posterior is sensitive in this regard, it is possible to bound $R(D)$ using the  Shannon lower bound. We proceed with some necessary definitions before we state the main results.
\begin{definition}[Discrepancy set]
\label{def:error.set}
    For model parameters $w,v \in \mathcal{W}$, let the {\em discrepancy set} $\mathcal{E}(w,v)$ 
    \begin{equation}
        \mathcal{E}(w,v) := \left\{x \in \mathcal{X}: \arg\max_{y} p(y|x, w) \neq \arg\max_y p(y|x, v) \right\},
    \end{equation}
    i.e. the set of points that are classified differently under $v$ than under $w$.
\end{definition}

Then, we define a notion of sensitivity with respect to $\mathcal{E}(w,v)$.
\begin{definition}[$\mu$-sensitivity in expectation]
\label{def:lower.continuous}
We say a model family is {\em $\mu$-sensitive in expectation} if there is a constant $0 < \mu < \infty$  such that for any %
absolutely continuous distribution $P_{V|W}$,  %
    \begin{equation}
        \mathbb{E}_{W,V}[\mathbb{P}_X[X \in \mathcal{E}(W, V)]] \geq \mu~ \mathbb{E}_{W,V}[\|W - V\|_\infty],
    \end{equation}
    where $\mathcal{E}(w,v)$ is the discrepancy set from Definition \ref{def:error.set} and the expectations are computed with respect to $P_{W,V} = P_W P_{V|W}$.
\end{definition}
Roughly speaking, sensitivity in expectation %
indicates that if we draw $V$ in the neighborhood of $W$, the volume of the error set $\mathcal{E}(W,V)$, measured via $P_X$ is bounded by $\mu\|W - V\|_\infty$ on {\em average} over $W$ and $V$. A sufficient condition for sensitivity is for the volume of $\mathcal{E}(w,w+u)$ to indeed be bi-Lipschitz continuous in perturbations $u$ for almost every $w \in \mathcal{W}$. Because the condition holds in expectation, it is further sufficient for the the volume of $\mathcal{E}(w,w+u)$ to be bi-Lipschitz over some region of $\mathcal{W}$ having nonzero probability in $P_W$. Neither of these conditions is strictly necessary.

For overparameterized models, $\mu$-sensitivity in expectation may not hold. In this setting, there may be non-trivial regions of $\mathcal{W}$ that are roughly invariant and have small discrepancy sets. In this case, it is possible to choose $P_{V|W}$ such that $\|V-W\|_\infty$ is large but $\mathcal{E}(W,V)$ is small. As we will show later, it may be possible to reparameterize such models to establish $\mu$-sensitivity in expectation.

We emphasize that some notion of sensitivity is essential in order to bound $R(D)$. Models that are insensitive to perturbations in model weights may be described with arbitrarily few bits.
Next, we derive a bound on the rate-distortion function for 0-1 loss. We bound $R(D)$ we include a bound for non-realizable models for completeness and because it will be useful for considering label noise in Section \ref{sect:margins}.

\begin{theorem}\label{thm:zero-one.rd}
    Let the loss function be the 0-1 loss, and suppose that the model class is $\mu$-sensitive in expectation. Then,
    \begin{equation}
        R(D) \geq \left[h(W) - d_w\log\left(\frac{2 e D}{(1-2L_{\max})\mu}\right)\right]^+,
    \end{equation}
    where $h(W)$ is the differential entropy of $W$, and where
    \begin{equation}
            L_{\max} := \mathbb{E}_{W}\left[ \sup_{x} [1 - \max_y p(y|x,W)]\right] \leq 1 - 1/|\mathcal{Y}|
    \end{equation}
    is the worst-case 0-1 error of the model family when the model parameters are known. Notice that $L_{\max}=0$ for realizable models, and hence if the model class is realizable, we obtain the following bound:
    \begin{equation}
        R(D) \geq \left[h(W) - d_w\log\left(\frac{2 e D}{\mu}\right)\right]^+.
    \end{equation}
\end{theorem}
While the full proof is in the appendix, the main idea is that the 0-1 excess loss is bounded by the probability of the error set (Lemma~\ref{lem:zero.one.loss}), which is in turn bounded by $\|w - v\|_\infty$ via the sensitivity assumption. Maximizing the entropy over the expected $L_\infty$ norm leads to an element-wise Laplace distribution for the encoder $Q_{\widehat{W} | W}$.

\subsection{Example: Intervals on $[0,1]$}\label{sect:intervals}
We illustrate the preceding bound with a simple one-dimensional example. Consider a binary classification problem in which $[0,1]$ is divided into two intervals at boundary point $w$:
 \begin{equation}
    p(y=1| x,w) = \mathbb{I}[x \geq w],
\end{equation}  
for $\mathcal{Y} = \{0,1\}$ and $\mathcal{X} = \mathcal{W} = [0,1]$. Further suppose that $P_X$ and $P_W$ are uniform and independent of each other. The discrepancy set is $\mathcal{E}(w,v) = [w,v) \cup [v,w)$, and it is straightforward to see that $\mathbb{P}_X[X \in \mathcal{E}(w,v)] = |w-v|$. Therefore, the model family is sensitive in expectation with $\mu=1$ and $p=1$, and in fact the inequality in Definition \ref{def:lower.continuous} holds with equality pointwise rather than in expectation. Further, $h(W) = 0$ because $W$ is uniform over $[0,1]$, and the rate-distortion function is bounded as
\begin{equation}
    R(D) \geq \left[\log\left(\frac{1}{2eD}  \right)\right]^+ .
\end{equation}

\subsection{Example: Half-spaces}\label{sect:half-spaces}
We next consider a problem in higher dimensions, in which case the parameterization of the problem is crucial to obtaining a bound on the rate-distortion function. Consider the family of half-space classifiers in $\mathbb{R}^{d}$, where a hyperplane\footnote{We suppose the hyperplanes pass through the origin for simplicity.} separates the classes:
\begin{equation}
    p(y=1| x,w) = \mathbb{I}[w^Tx \geq 0].
\end{equation}
Again $\mathcal{Y}= \{0,1\}$, but here $\mathcal{X}=\mathcal{W} = \mathbb{R}^{d}$. Further, $P_X$ and $P_W$ are taken to be independent and isotropic, i.e. invariant to rotations. It is straightforward to verify that the discrepancy set for this model is
\begin{equation}
    \mathcal{E}(w,v) = \{x \in \mathcal{X} : \mathrm{sgn}(w^Tx) \neq \mathrm{sgn}(v^Tx) \},
\end{equation}
i.e. the points for which the sign of the inner products differ, and its volume w.r.t $P_X$ is
\begin{equation}
    \mathbb{P}[X \in \mathcal{E}(w,v)] = \frac{1}{\pi}\cos^{-1}\left(\frac{w^Tv}{\|w\|\|v\|}\right).
\end{equation}
However, this model is {\em not} sensitive in expectation. Because the classification only depends on the direction of $x$ and $w$, it is possible to choose $P_{V|W}$ to ``trace out'' a region of $\mathcal{W}$ where the norm of $V$ and $W$ are different, but their angles are the same. This choice has a discrepancy set $\mathcal{E}(w,v)$ that remains the same size, even though $\| W - V\|_\infty$ gets arbitrarily large in expectation.

However, we can reparameterize this model family in such a way that sensitivity in expectation holds. Let $\widetilde{\mathcal{W}} = [-1,1]^{d-1}$ parameterize the set of unit-norm vectors in $\mathbb{R}^d$ via the mapping $f(w) = (w_1, \dots w_{d-1}, 1-\sqrt{\sum_{i=1}^{d-1} w_i^2})$. Then, the size of the discrepancy set is just
\begin{equation}
    \mathbb{P}[X \in \mathcal{E}(w,v)] = \frac{1}{\pi}\cos^{-1}\left(w^Tv)\right),
\end{equation}
 with the final elements of $w$ and $v$ computed via the mapping $f$. Because we have normalized $w$ and $v$, there is no invariant region of $\widetilde{\mathcal{W}}$ that we can trace out and we get $\mu$-sensitivity in expectation. 
\begin{lemma}
The reparameterized model class satisfies $\mu$-sensitivity in expectation with $\mu =\pi^{-1}.$ That is
\begin{equation}
     \mathbb{P}[X \in \mathcal{E}(w,v)] \geq \frac{1}{\pi} \| w-v \|_\infty.
\end{equation}
\label{lemma:half-space-discrepancy}
\end{lemma} 
Hence, the rate-distortion function is bounded as 
\begin{equation}
    R(D) \geq \left[(d-1)\log\left(\frac{1}{4\pi eD}\right)\right]^+.
\end{equation}
In this case, the natural parameterization does not capture the number of free parameters in the model family, and it is required to describe only $(d-1)$ parameters with sufficient fidelity to get a high-accuracy approximation of the model with respect to the 0-1 loss.

\section{Bounding $I(Z^n;W)$ for Realizable Models}
\label{sec:bounds-on-MI}

Here, we present new techniques for bounding the mutual information for realizable models. To do so, we need a notion of consistency of model parameters with respect to the training set.
\begin{definition}[Consistency set]
\label{def:consistency}
    For fixed training set $z^n \in \mathcal{Y}^n \times \mathcal{X}^n$, let the {\em consistency set} $\mathcal{C}(z^n)$ be the set of all models consistent with the training set:
    \begin{equation}
        \mathcal{C}(z^n) := \{w \in \mathcal{W}: \forall i\in [n] \quad f(x_i,w) \in \mathcal{F}_{y_i}\}.
    \end{equation}
\end{definition}

\begin{lemma}\label{lem:realizable.mi}
    Let $P_{ZW} = P_{Y|XW}P_{X|W}P_W$ be a realizable model family. Suppose further the special case $P_{X|W} = P_W$; i.e. the model parameters specify the classifier of $Y$ given $X$, but not the marginal distribution on data points. Then, $I(Z^n; W) = H(Y^n | X^n)$, which further can be expressed as
    \begin{equation}
        I(Z^n; W) = -\mathbb{E}_{Z^n} [ \log (\mathbb{P}_{V \sim P_W}[V \in \mathcal{C}(Z^n)]) ],
    \end{equation}
    where $\mathcal{C}(z^n)$ is the consistency set from Definition \ref{def:consistency}.
\end{lemma}
In other words, if the (log) volume of $\mathcal{C}(Z^n)$ is small on average, the training set eliminates all but a small set of possible model parameters $W$ and provides high information. The full proof is in the appendix. In some cases it is possible to bound $I(Z^n;W)$ by direct analysis of $\mathbb{P}_{V \sim P_W}[V \in \mathcal{C}(Z^n)]$. For VC classes, however, we present a simple outer bound entirely in terms of $n$ and the VC dimension.
\begin{lemma}\label{lem:vc.mi}
    Let $P_{ZW} = P_{Y|XW}P_{X}P_W$ be a realizable model family such that $p(y|x,w)$ induces a classification rule with finite VC dimension $\dvc$. Then,
    \begin{equation}
        I(Z^n; W) \leq \dvc \log(e \cdot n),
    \end{equation}
    which for $n \geq \dvc$ can be sharpened to
    \begin{equation}
        I(Z^n; W) \leq \dvc \log\left(\frac{e \cdot n}{\dvc}\right).
    \end{equation}
\end{lemma}
\begin{proof}
    By Lemma \ref{lem:realizable.mi}, $I(Z^n; W) = H(Y^n| X^n)$. For a model with finite VC dimension, Sauer's lemma \citep{sauer1972density} bounds the maximum number of labelings of $X^n$:
    \begin{equation}
        |\{y^n \in \{0,1\}^n : p(y^n|x^n) \neq 0 \}| \leq \sum_{k=0}^{\dvc} \binom{n}{k},
    \end{equation}
    which is bounded by $e \cdot n^{\dvc}$ for $n \geq 1$ and $(en/\dvc)^{\dvc}$ for $n \geq \dvc$. Further, the conditional entropy is maximized by a uniform distribution over all valid $Y^n$, which yields the result.
\end{proof}

In the case of a realizable model that is both a VC class and is $\mu$-sensitive in expectation, we obtain a lower bound on the 0-1 excess risk.
\begin{corollary}\label{cor:vc.risk}
    Let $P_{ZW} = P_{Y|XW}P_{X}P_W$ be a realizable model family such that $p(y|x,w)$ induces a classification rule with finite VC dimension $\dvc$, and suppose $P_{ZW}$ is $\mu$-sensitive in expectation per Definition \ref{def:lower.continuous}. Then, the 0-1 excess Bayes risk for any parametric learning rule is bounded as
    \begin{equation}
        E_n \geq \left(\frac{d}{e \cdot n}\right)^{\frac{\dvc}{d_w}} \cdot \left(\frac{\mu}{2e}\right)\exp(h(W)/d_w),
    \end{equation}
    for $n \geq \dvc$. In particular, when $d_w = \dvc$, the bound is $\Omega(\dvc/n)$.
\end{corollary}

In other words, when the VC dimension matches the number of free parameters in the model family, we obtain a $\Omega(\dvc/n)$ bound on the excess risk. This matches (up to possible log factors) the upper bounds of \cite{tsybakov2004optimal,SteinkeZ20}, which are known to be minimax tight under mild conditions. This result shows that, under these conditions, they are also tight in the average case over the model family. Further, as we show in Section \ref{sect:mer}, this bound matches an {\em outer bound} on the minimum excess risk for realizable models.

Somewhat paradoxically, a VC dimension that overestimates the model complexity results in a lower bound on the Bayes risk that is too {\em optimistic}, whereas in classical VC theory it results in a pessimistic upper bound. This is because our estimate of the mutual information $I(Z^n;W)$ is increasing in the VC dimension: a higher VC dimension means more bits per sample learned about the true model parameters $W$.

More generally, let the {\em mutual information dimension} be defined as 
\begin{equation}
    d_I := \lim_{n \to \infty} \frac{I(Z^n; W)}{\log n}.
\end{equation}
By Lemma~\ref{lem:vc.mi}, $d_I \leq d_{\text{vc}}$. This bound is not always tight, as we shall see in the example in Section~\ref{ex:MI-half-space}.
It is straightforward to adapt Corollary \ref{cor:vc.risk} to see that for mutual information dimension $d_I$ the excess risk scales as
 $\Omega\left(\left(\frac{d_w}{n}\right)^{\frac{d_I}{d_w}}\right).$ It seems that the mutual information dimension is a fundamental quantity in understanding the  Bayes risk.
 
Indeed, in Appendix~\ref{app:smooth.models} we observe that for (non-realizable) smooth model families (where $P_{Y|XW}$ is a smooth function), the mutual information dimension is $d_I =\frac{d_w}{2}$. This results in Bayes risk lower bounds of the form $\Omega\left(\sqrt{\frac{d_w}{n}}\right)$, as opposed to the fast $\Omega\left(\frac{d_w}{n}\right)$ bounds that we obtain for realizable models.

\subsection{Example: Intervals on $[0,1]$}
We revisit the simple setting of Section \ref{sect:intervals}, in which the classifier divides $[0,1]$ into two intervals. This classifier family shatters only one point, so $\dvc=1$ and $I(Z^n;W) \leq \log(e \cdot n)$. The resulting risk bound is $E_n \geq n/2$. By contrast, Theorems 1.2 and 1.6 of \cite{SteinkeZ20} give an outer bound on the expected generalization error for VC classes, which corresponds to the upper bound $E_n \leq (3\log n + 6)/n$. These bounds match up to constants and a $\log n$ term as expected.

\subsection{Example: Half-spaces}
\label{ex:MI-half-space}
Next we revisit the family of half-space classifiers considered in Section \ref{sect:half-spaces}. It is straightforward to verify that half-spaces, constrained to have the origin on the boundary plane, have VC dimension $\dvc = d$. As a result, $I(Z^n;W) \leq d \log(e\cdot n/d)$, giving a bound on the 0-1 risk that scales as
\begin{equation}
    E_n = \Omega\left(\frac{d}{n} \right)^{\frac{d}{d-1}}.
\end{equation}
By contrast, the VC results of \cite{SteinkeZ20} give the upper bound
\begin{equation}
\label{eq:steinke-halfspace}
    E_n \leq \frac{3d \log(n) + 6}{n}.
\end{equation}
For large $d$, the bounds are approximately tight up to a log term. Otherwise, however, the bounds have a substantial gap. Indeed, for $d=2$ the lower bound predicts $E_n = \Omega(1/n^2)$, which is much faster than standard results suggest.

The discrepancy is due to the fact that the bound on $R(D)$ is proportional to $(d-1)$, whereas the bound on $I(Z^n;W)$ scales as $d$. (At least) three explanations are possible: (1) The slack compared to the frequentist bounds is an artifact of the information-theoretic bounding technique based on rate-distortion analysis, (2) the bound on $R(D)$ is loose, and (3) the bound on $I(Z^n;W)$ is loose. We conjecture that the third explanation is correct, and that the mutual information dimension is equal to $(d-1)$ for this problem. We verify this conjecture for $d=2$.

To see this, consider the following parameterization of the family of halfspaces in $\mathbb{R}^2$:
\begin{equation}
    p(y=1 | x, w) = \mathbb{I}\left[\angle x \geq w \right],
\end{equation}
where $\angle x$ is the normalized (to $[0,1)$) angle of $x$ off of the point $(1,0)$, i.e. a point $X$ is classified with $Y=1$ if its normalized angle is greater than $w$. We take $W$ to be uniform over $[0,1)$ and continue to take $P_X$ to be isotropic. 

We bound $I(Z^n;W)$ by direct analysis of the consistency set $C(z^n) \subset \mathcal{W}$. The set of model parameters $w$ consistent with a fixed training set $z^n$ is the interval between classes in the training set:
\begin{equation}
    \mathcal{C}(z^n) = [\theta_-, \theta_+ ),
\end{equation}
where $\theta_-$ is maximum angle of the sample $X_i$ with $Y_i=1$, and $\theta_+$ is the minimum angle of the sample $X_i$ with $Y_i=0$. Then, for simplicity we bound the probability of $\mathcal{C}(z^n)$ in terms of the {\em smallest} angle:
\begin{equation}
    \mathbb{P}_W[W \in \mathcal{C}(Z^n)] \geq \mathbb{P}[W \leq \theta_*],
\end{equation}
where $\theta_* = \arg\min_i \angle X_i$ is the smallest angle in the training set. The distribution of the minimum of $n$ standard uniform random variables has a beta distribution $\beta(1,n)$. Further, the expected logarithm of the beta distribution is known in closed form. Putting these together with Lemma \ref{lem:realizable.mi}, we obtain
\begin{equation}
    I(Z^n;W) \leq -\mathbb{E}_{Z^n}[\mathbb{P}[W \leq \theta_*]] = \psi(n) + \gamma \leq \log(n) + \gamma,
\end{equation}
where $\psi(x)$ is the digamma function and $\gamma$ is the Euler–Mascheroni constant. This tighter bound for the case $d=2$ results in the lower bound
\begin{equation}
    E_n \geq \frac{1}{4\pi e^{\gamma+1} n},
\end{equation}
which scales as $E_n = \Omega(1/n)$ and agrees with the upper bound by~\citep{SteinkeZ20} in \eqref{eq:steinke-halfspace} up to a log factor. In this case, the VC bound overestimates the complexity of the model family, which results in a too-optimistic lower bound on the 0-1 Bayes risk.

\section{Plug-in Classifiers vs. Minimum Excess Risk}\label{sect:mer}
As mentioned in Section \ref{sect:bayes.learning}, we study the expected risk of {\em plug-in} classifiers $p(y|x,\widehat{w})$, which are not in general optimum. By contrast, \citet{XuR20} and \citet{KolahiMKB21} study the {\em minimum excess risk} over all possible classifiers $\psi: \mathcal{X} \times \mathcal{Z}^n \to \mathcal{Y}$ that map a test point and the training set to a classification. Given a loss function $\ell: \mathcal{Y} \times \mathcal{Y} \to \mathbb{R}$, the minimum excess risk is given by
\begin{equation}
    \mathrm{MER}_n := \inf_\psi \mathbb{E} \ell(Y,\psi(X,Z^n)) - L^*,
\end{equation}
where $L^*$ is again the Bayes risk if $W$ is known. The Bayes-optimum classifier is usually a function of the posterior $p(y | x, z^n)$. \citet{XuR20} give upper bounds on the MER for a variety of loss functions, and \cite{KolahiMKB21} give lower bounds for quadratic loss using rate-distortion methods similar to the ones used in this paper. A natural question is the relationship between the MER and the plug-in excess risk $E_n$ studied in this paper.

First, $\mathrm{MER}_n \leq E_n$. This is true because any plug-in classifier $p(y|x,\widehat{w})$ can be realized by a Bayes classifier $\psi$, but not vice versa. Therefore, the Bayes risk lower bounds that we derive in this paper are {\em not} in general lower bounds on the MER. However, $\mathrm{MER}_n$ and $E_n$ often have the same scaling laws. Indeed, for smooth model families, \cite{XuR20} shows $\mathrm{MER}_n = O(d/n)$ for the log loss and $\mathrm{MER}_n = O(\sqrt{d/n})$ for 0-1 loss. We obtain ``matching'' lower bounds on $E_n$ for models that are $\mu$-sensitive in expectation in Section \ref{sect:smooth.mi.results}. This suggests---but of course does not prove---that the MER and the plug-in excess risk have similar scaling behavior.

Furthermore, we can show a similar correspondence for realizable models under 0-1 loss. \citet{XuR20} give a general $O(\sqrt{d/n})$ bound on MER for 0-1 loss, but they do not consider realizable models or whether fast rates are possible. Here we obtain fast rates for realizable models that are also VC classes.
\begin{theorem}
    Let $P_{ZW}$ be a realizable model family such that $p(y|x,w)$ induces a classification rule with finite VC dimension $\dvc$. Then, for $n \geq \dvc$, the minimum excess risk satisfies
    \begin{equation}
        \mathrm{MER}_n \leq \frac{3 \dvc \log(en/\dvc)}{n}.
    \end{equation}
\end{theorem}
\begin{proof}
    First, Lemma 3 of \cite{KolahiMKB21} shows that for the 0-1 loss,
    \begin{equation}
        \mathrm{MER}_n \leq 2L^* + 3I(Y; W | X, Z^n).
    \end{equation}
    Because the model is realizable, $L^* = 0$. Further, the proof of Theorem 2 of \cite{XuR20} shows that $I(Y; W | X, Z^n) \leq \frac{1}{n}I(Z^n;W)$, so
    \begin{equation}
        \mathrm{MER}_n \leq \frac{3I(Z^n;W)}{n}.
    \end{equation}
    Finally, Lemma \ref{lem:vc.mi} shows that $I(Z^n;W) \leq \dvc \log(en/\dvc)$, which establishes the result.
\end{proof}

We conjecture that there are mild regulatory conditions under when $\mathrm{MER}_n$ and $E_n$ agree order-wise. Studying the gap between them is a fruitful area for future work.

\section{Noisy Labels}\label{sect:margins}
Under appropriate conditions, it is possible to achieve ``fast'' $O(1/n)$ rates even if the underlying data distribution is not perfectly realizable, i.e. the labels are noisy, so long as suitable conditions on the data distribution are satisfied, such as the so-called Tsybakov-Mammen margin condition \citep{mammen1995asymptotical,tsybakov2004optimal} and the Bernstein condition \citep{bartlett2006empirical}. These conditions ensure that the distribution is ``nearly'' realizable in a precise sense; see \citep{van2015fast} for a detailed explanation of these types of conditions.

In this section we study the impact of label noise on the excess Bayes risk lower bounds. We consider two label noise scenarios: (1) the underlying $P_{ZW}$ {\em is} realizable, but due to annotation errors the training set contains label noise, and (2) the underlying data distribution $P_{ZW}$ has noisy labels, so noise is present in both the train and test sets. In the first case, we provide an exact expression for the penalty due to label noise, which under mild conditions is bounded by a constant for any $n$ and decays to zero as $n \to \infty$. In the second case, we impose the Tsybakov-Mammen condition and show that the lower bound exhibits a margin-dependet scaling similar to the upper bound of  \citep{tsybakov2004optimal}. 

\subsection{Noisy Training Data, ``Clean'' Test Sets}
Under noisy training labels, we formalize the Bayesian learning problem as follows. Let $P_{Z|W}$, $P_W$ specify the model family (with clean labels) and prior as before, and let $P_{\widetilde{Z}|W} = P_{X,\widetilde{Y} | W}$ specify the data distribution of the model family with noisy labels. We suppose the Markov chain $W \to Y \to \widetilde{Y}$ to rule out the possibility that the label noise ``leaks'' information about $W$ to the learner; otherwise the label noise distribution is arbitrary. Then, a learning rule $Q_{\widehat{W} | \widetilde{Z}^n}$ estimates $W$ from the noisy training set. The excess loss and Bayes risk for the learning rule $Q_{\widehat{W} | \widetilde{Z}^n}$ are computed according to the clean distributions $P_{Z|W}$, $P_W$.

In this case, it is straightforward to particularize Theorem \ref{thm:dpi} and show that a learning rule $Q_{\widehat{W} | \widetilde{Z}^n}$ has excess Bayes risk less than or equal to $D$ only if
\begin{equation}
    R(D) \leq I(\widetilde{Z}^n; W),
\end{equation}
where $R(D)$ is computed as before using $P_{Z|W}$, $P_W$. Notice that $I(\widetilde{Z}^n; W) \leq I(Z^n; W)$ due to the data processing inequality, so $R(D) \leq I(Z^n; W)$ remains a necessary condition on the Bayes risk.

To understand the impact of noisy labels on the Bayes risk, then, we need to examine the difference between the mutual information $I(Z^n;W)$ provided by clean training data and the information $I(\widetilde{Z}^n;W)$ provided by noisy labels. 

In the realizable case, we can specify this gap exactly.
\begin{lemma}\label{lem:noisy.mi}
    Let $P_{Z,W}$ be a binary realizable model with $P_{X|W}=P_X$, and define the ``nearly'' realizable model $P_{\widetilde{Z},W} = P_{X,\widetilde{Y},W}$, where noisy labels $\widetilde{Y} \in \mathcal{Y}$ satisfy the Markov chain $W \to Y \to \widetilde{Y}$. Let $U^n := Y^n \oplus \widetilde{Y}^n$ be the label noise. Then,
    \begin{equation}
        I(Z^n;W) - I(X^n,\widetilde{Y}^n; W) = H(U^n | X^n, \widetilde{Y}^n),
    \end{equation}
\end{lemma}
\begin{proof}
    By the chain rule of mutual information,
    \begin{equation}
        I(X^n, Y^n, \widetilde{Y}^n;W) = I(X^n, \widetilde{Y}^n; W) + I(X^n, Y^n; W| \widetilde{Y}^n).
    \end{equation}
    Further, $I(X^n, Y^n, \widetilde{Y}^n;W) = I(Z^n;W)$ by the Markov chain $W \to (X^n, Y^n) \to \widetilde{Y}^n$, so the quantity of interest is
    \begin{equation}
        I(Z^n;W) - I(X^n,\widetilde{Y}^n; W) = I(X^n, Y^n; W| \widetilde{Y}^n).
    \end{equation}
    Again applying the chain rule of mutual information,
    \begin{equation}
         I(X^n, Y^n; W| \widetilde{Y}^n) = I(X^n; W) +  I(Y^n; W| X^n, \widetilde{Y}^n) = I(Y^n; W| X^n, \widetilde{Y}^n),
    \end{equation}
    where the latter equality is because $X$ and $W$ are independent. Finally,
    \begin{align}
        I(Y^n; W| X^n, \widetilde{Y}^n) &= H(Y^n | X^n, \widetilde{Y}^n) - H(Y^n | W, X^n, \widetilde{Y}^n) \\
        &= H(Y^n | X^n, \widetilde{Y}^n) \\
        &= H(U^n |  X^n, \widetilde{Y}^n),
    \end{align}
    where the second equality is because $P_{ZW}$ is realizable, and the final equality is because $Y^n$ and $U^n$ (given $\widetilde{Y}^n$) have the same distribution  up to a relabeling.
\end{proof}
The information gap is exactly the amount of uncertainty left in the label noise if we have access to both $X^n$ and the noisy labels $\widetilde{Y}^n$. For a simple model family and/or for large $n$, there will be relatively few noiseless labelings $Y^n$ consistent with $X^n$, and it will be possible to detect most of the label errors and $H(U^n | X^n, \widetilde{Y}^n)$ will be small. Otherwise the gap may be large.

Indeed, for a model family that is learnable from $(X^n,\widetilde{Y}^n)$---which includes any VC class under i.i.d. label flips with probability less than $1/2$---one learns $W$ arbitrarily well with enough samples, which implies that $Y^n$ and thus the label noise is known arbitrarily as well. In this case $H(U^n | X^n, \widetilde{Y}^n) \to 0$, and the entropy is bounded above by a constant. This further implies an excess risk lower bound that differs from the noiseless case by a constant that scales as $1 + o_n(1)$.

This result suggests that, for realizable models, even moderate label noise has small fundamental impact on learning rates. An important question for future work is whether tighter {\em outer} bound are possible in the realizable case, as well as precise convergence rates on $H(U^n | X^n, \widetilde{Y}^n)$.

\subsection{Noisy Training and Test Sets}
When both training and test sets are corrupted by label noise, the rate-distortion function and the mutual information in Theorem \ref{thm:dpi} change. Similar to before let $P_{Z,W}$, be a {\em realizable} model family, and let $P_{\widetilde{Z},W}$ specify the distribution of the model family with label noise. Again we suppose the Markov chain $W \to Y \to \widetilde{Y}$, i.e. noise is added to labels $Y$ independent of $W$. Let $Q_{\widehat{W} | \widetilde{Z}^n}$ be a parametric learning rule. Because both train and test sets are drawn according to $P_{\widetilde{Z},W}$, the excess loss and Bayes risk are computed according to the noisy distribution.

A straightforward application of Theorem \ref{thm:dpi} to this scenario shows that we must have $I(\widetilde{Z}^n; W) \geq \widetilde{R}(D)$ for any learning rule with excess risk less than or equal to $D$, where $\widetilde{R}(D)$ is the rate-distortion function computed according to $P_{\widetilde{Z},W}$. Further, by the data-processing inequality, $I(\widetilde{Z}^n; W) \leq I(Z^n; W)$ and we must also have $I(Z^n; W) \geq \widetilde{R}(D)$. In this case, we can get valid bounds on the excess risk by considering only the changes to the rate-distortion function.

We can bound $\widetilde{R}(D)$ when $P_{\widetilde{Z},W}$ satisfies the Tsybakov-Mammen condition, which we state below in a form adapted to our problem setting.
\begin{definition}[Tsybakov-Mammen condition with margin $t$]
    We say that a binary model family $P_{\widetilde{Z},W}$ satisfies the {\em Tsybakov-Mammen} condition with margin $t$ if, for every $x \in \mathcal{X}$ and $w \in \mathcal{W}$, there is an $h > 0$ such that
    \begin{equation}
        |2p(y|x,w) - 1| \geq t.
\end{equation}
\end{definition}
The Tsybakov-Mammen condition enforces that the posterior probability for every point is bounded away from $1/2$; larger $t$ leads to a larger margin. Notice that the Tsybakov-Mammen condition with margin $t=1$ recovers the realizable model family (Definition~\ref{def:realizable}) that we have investigated in this paper.

If $P_{Z,W}$ is also a VC class with VC dimension $\dvc$, \cite{tsybakov2004optimal} shows that when Tsybakov-Mammen condition with margin $t$ is satisfied, the excess risk is upper bounded as
\begin{equation}
    E_n = O\left(\frac{\dvc}{nt}\log\left(\frac{n t^2}{\dvc} \right)\right),
\end{equation}
and that this bound is minimax tight. The conditional mutual information bounds of \citet{SteinkeZ20} and \citet{grunwald2021pac} give similar scaling laws under slightly different conditions, again showing that fast rates are possible under well-behaved label noise. A natural question is whether we can derive similar {\em lower} bounds in this setting.

\begin{theorem}\label{thm:Tsybakov.lower.bound}
    Let $P_{Z,W}$ be a realizable binary model for which $P_{Y|X,W}$ is a VC class with dimension $\dvc$, and let $P_{\widetilde{Z},W} = P_{X,\widetilde{Y},W}$ be an ``almost'' realizable model that satisfies the conditions of Lemma \ref{lem:noisy.mi} and the Tsybakov-Mammen condition for some $t > 0$. Finally, suppose that $P_{\widetilde{Z},W}$ is $\mu$-sensitive in expectation. Then, the excess Bayes risk is bounded below as
    \begin{equation}\label{eqn:tsybakov.lower.bound}
        E_n \geq   \frac{\mu t}{2e}\left( \frac{d}{e \cdot n}\right)^{\frac{\dvc}{d_w}}\exp(h(W)/d_w).
    \end{equation}
\end{theorem}
\begin{proof}
    For a model satisfying the Tsybakov-Mammen condition with margin $t$, we have $L_{\max} \leq \frac{1- t}{2}$. Applying this to Theorem \ref{thm:zero-one.rd}, we obtain
    \begin{equation}
        \widetilde{R}(D) \geq \left[h(W) - d_w\log\left(\frac{2eD}{\mu t}\right)\right]^+.
    \end{equation}
    Applying Theorem  \ref{thm:dpi} and Lemma \ref{lem:vc.mi} gives the result.
\end{proof}

We obtain a lower bound that {\em decreases} with decreasing margin: the lower bound scales linearly in $t$, whereas the Tsybakov upper bound scales in $t^{-1}$. A potential explanation for this difference is that we have used $I(Z^n;W)$ to bound the mutual information provided by the training set. Indeed, $\widetilde{R}(D) \leq R(D)$ and $I(\widetilde{Z}^n; W) \leq I(Z^n;W)$, and perhaps a more careful analysis of the mutual information might result in a different scaling law. However, per Lemma \ref{lem:noisy.mi}, the gap between $I(\widetilde{Z}^n; W)$ and $I(Z^n;W)$ is $o(1)$ in the realizable case, so for large $n$ this does not provide an explanation.

Instead, we point out that smaller margin can result in a problem that is easier with respect to the {\em excess} risk. For example, a realizable model corrupted with i.i.d. random flips at probability $p \approx 1/2$ has margin $t \approx 0$ and oracle Bayes risk of $p \approx 1/2$ . In this case, even a zero-information classifier that outputs random labels achieves Bayes risk $1/2$ and therefore small Bayes risk, which roughly agrees with (\ref{eqn:tsybakov.lower.bound}). The bound on $\widetilde{R}(D)$ derived in the proof of Theorem \ref{thm:Tsybakov.lower.bound} accounts in part for such cases. Future work involves a more detailed analysis of the tightness of this bound.

\section{Conclusion}
\vspace{-.07in}
We have derived fundamental limits on the performance of interpolating classifiers by bounding the excess Bayes risk of realizable model families, showing that models that satisfy a Lipschitz-like continuity in their decision regions have average 0-1 loss decaying as $\Omega(1/n)$. Our bounds also serve as lower bounds on generalization error, and they establish that ``fast'' convergence rates can be tight in the average case.
\bibliography{alt.bib}

\newpage
\appendix

\section{Proofs of results from main text}

\subsection{Proofs for Section \ref{sect:bayes.learning}}
\renewcommand*{\proofname}{Proof of Lemma \ref{lem:risk.and.generalization}.}
\begin{proof}
This is a straightforward observation from the definitions of realizable models and interpolating models.
\end{proof}
\renewcommand*{\proofname}{Proof.}

\subsection{Proofs for Section \ref{sect:rd.bounds}}
\renewcommand*{\proofname}{Proof of Theorem \ref{thm:dpi}.}
\begin{proof}
    We first observe that the model parameters $W$, training set $Z^n$, and learned model parameters $\widehat{W}$ satisfy the Markov chain $W \to Z^n \to \widehat{W}$, due to the factorization $P_{Z^n, W, \widehat{W}} = Q_{\widehat{W} | Z^n}P_{Z^n | W}P_W$. Then, by the data processing inequality,
    \begin{equation}
        I(\widehat{W} ; W) \leq I(Z^n ; W).
    \end{equation}
    Minimizing $I(\widehat{W} ; W)$ subject to the constraint on the excess Bayes risk gives the result.
\end{proof}
\renewcommand*{\proofname}{Proof.}

\subsection{Proofs for Section \ref{sec:bounds-on-rate-distortion}}

The following result relates the excess loss to the discrepancy set.
\begin{lemma}[0-1 Excess loss]\label{lem:zero.one.loss}
    Let
    \begin{equation}
        L(w) := \sup_x \left[1 - \max_y p(y|x,w)\right] \leq 1-1/|\mathcal{Y}|
    \end{equation}
    denote the worst case 0-1 error of the ML classifier when the true model parameters $w$ are fixed and known.
    Then, the excess loss $D(w || v)$ is lower and upper bounded by
    \begin{equation}
      (1 - 2 L(w)) \mathbb{P}[X \in \mathcal{E}(w,v)] \leq   D(w || v) \leq \mathbb{P}[X \in \mathcal{E}(w,v)].
    \end{equation}
    In the realizable case, $L(w)=0$, and we can express the excess loss exactly:
    \begin{equation}
         D(w || v) = \mathbb{P}[X \in \mathcal{E}(w,v)].
    \end{equation}
\end{lemma}
\renewcommand*{\proofname}{Proof of Lemma \ref{lem:zero.one.loss}.}
\begin{proof}
    The realizable case is immediate: the true classifier $p(y|x,w)$ makes no mistakes, so the 0-1 error is equal to the probability of a data point $x$ such that $p(y|x,v)$ gives a different classification.
    
    For the lower bound, observe that for fixed $x$ and $w$, the expected 0-1 loss of a classifier using $p(y|x,v)$ is
    \begin{align}
         1 - \max_y p(y|x,v) &=
        \mathbb{I}[\arg\max_y p(y|x,v) \neq \arg\max_y p(y|x,w)] \cdot \max_y p(y|x,w)\nonumber\\
        & ~~~+ \mathbb{I}[\arg\max_y p(y|x,v) = \arg\max_y p(y|x,w)] \cdot (1 - \max_y p(y|x,w))
    \end{align}
    This observation follows from the fact that there are two possible ways we can make an error using $p(y|x,v)$: $p(y|x,v)$ makes a {\em different} classification than $p(y|x,w)$ and $p(y|x,w)$ is {\em correct}, or $p(y|x,v)$ makes the {\em same} classification as $p(y|x,w)$ but $p(y|x,w)$ is {\em incorrect}. %
    Therefore, the excess loss $D(w || v)$ can be rewritten as
     \begin{align}
        D(w || v) &=  \mathbb{E}_X [1 - \max_y p(y|X,v)] - \mathbb{E}_X [1 - \max_y p(y|X,w)] \\
        & = \mathbb{E}_X \left[\mathbb{I}[\arg\max_y p(y|X,v) \neq \arg\max_y p(y|X,w)] \cdot \max_y p(y|X,w)\right] \nonumber\\
        & ~~~+ \mathbb{E}_X \left[ \mathbb{I}[\arg\max_y p(y|X,v) = \arg\max_y p(y|X,w)] \cdot (1 - \max_y p(y|X,w))\right] \nonumber\\
        & ~~~- \mathbb{E}_X \left[1 - \max_y p(y|X,w)\right] \\
        & = \mathbb{E}_X \left[\mathbb{I}[\arg\max_y p(y|X,v) \neq \arg\max_y p(y|X,w)] \cdot \max_y p(y|X,w)\right] \nonumber\\
        & ~~~ - \mathbb{E}_X \left[ \mathbb{I}[\arg\max_y p(y|X,v) \neq \arg\max_y p(y|X,w)] \cdot (1 - \max_y p(y|X,w))\right].
\end{align}
To obtain the lower bound, notice that
\begin{align}
       D(w || v)  & \geq  \mathbb{E}_X \left[\mathbb{I}[\arg\max_y p(y|X,v) \neq \arg\max_y p(y|X,w)]\right] \times (1- L(w)) \nonumber\\
        & ~~~-  \mathbb{E}_X \left[ \mathbb{I}[\arg\max_y p(y|X,v) \neq \arg\max_y p(y|X,w)]\right]\times  L(w) \\
        & = \mathbb{P}_X[X \in \mathcal{E}(w,v)] (1- 2L(w)),
    \end{align}
which completes the lower bound. On the other hand,
\begin{align}
       D(w || v)  & \leq  \mathbb{E}_X \left[\mathbb{I}[\arg\max_y p(y|X,v) \neq \arg\max_y p(y|X,w)]\right] \times 1  \nonumber\\
        & ~~~-  \mathbb{E}_X \left[ \mathbb{I}[\arg\max_y p(y|X,v) \neq \arg\max_y p(y|X,w)]\right] \times 0 \\
        & = \mathbb{P}_X[X \in \mathcal{E}(w,v)] ,
    \end{align}
which establishes the upper bound.
\end{proof}
\renewcommand*{\proofname}{Proof.}

\renewcommand*{\proofname}{Proof of Theorem \ref{thm:zero-one.rd}.}
\begin{proof}
    First, we invoke the usual Shannon lower bound on $R(D)$:
    \begin{equation}
        R(D) \geq h(W) - \sup_{U} h(U),
    \end{equation}
    where the supremum is over all $U$ satisfying the distortion constraint $D(W\|W+U) \leq D.$ Invoking Lemma~\ref{lem:zero.one.loss} and taking the expectation, we can relax the constraint to:
    \begin{equation}
   (1-2 L_{\max})\mathbb{E}_{W,U}[ \mathbb{P}( X \in \mathcal{E}(W,W+U))]  \leq D.
    \label{eq:non-realizable-K}
    \end{equation}
    By the assumption of $\mu$-sensitivity in expectation, this constraint can be relaxed to
    \begin{equation}
        \mathbb{E}[\| U \|_\infty] \leq \frac{D}{(1-2L_{\max}) \mu}.
    \end{equation}
    It is well known that the element-wise Laplace distribution maximizes entropy with respect to a constraint on the expected $L_\infty$ norm. Choosing the scale factor to match the constraint and computing the resulting entropy yields the claim. 
    In the realizable case, we notice that~\eqref{eq:non-realizable-K} would still hold if we set $L_{\max} = 0$.
\end{proof}
\renewcommand*{\proofname}{Proof.}

\renewcommand*{\proofname}{Proof of Lemma~\ref{lemma:half-space-discrepancy}.}
\begin{proof}
    The proof follows from these steps: 
    \begin{align}
         \mathbb{P}[X \in \mathcal{E}(w,v)] &= \frac{1}{\pi}\cos^{-1}\left(w^Tv)\right) \\
         & =  \frac{1}{\pi}\cos^{-1}\left(\frac{\|w\|_2^2 + \|v\|_2^2 - \|w - v\|_2^2}{2})\right)\\
         & =  \frac{1}{\pi}\cos^{-1}\left(1 - \frac{1}{2} \|w - v\|_2^2)\right)\\
         & \geq \frac{1}{\pi} \| w-v \|_2 \label{eq:arccos-bound}\\
         & \geq \frac{1}{\pi} \| w-v \|_\infty,
    \end{align}
where~\eqref{eq:arccos-bound} follows from the fact that $\cos^{-1}\left(1-\frac{1}{2}x^2\right) \geq x$ for $x \in [0, 2].$
\end{proof}
\renewcommand*{\proofname}{Proof.}

\subsection{Proofs of Section~\ref{sec:bounds-on-MI}}

\renewcommand*{\proofname}{Proof of Lemma \ref{lem:realizable.mi}.}
\begin{proof}
    The first part is a consequence of the chain rule for mutual information:
    \begin{align}
        I(W; Z^n) &= I(W; X^n) + I(W; Y^n | X^n) \\
        &= H(Y^n | X^n) - H(Y^n | X^n, W) \\
        &= H(Y^n | X^n),
    \end{align}
    where $I(W; X^n) = 0$ because $X_i$ and $W$ are independent, and
    where $H(Y^n | X^n, W) = 0$ due to realizability; given the model parameters and data samples, the labels are determistic. To get the second expression, observe that
    \begin{align}
        p(y^n | x^n) &= \mathbb{E}_W [p(y^n | x^n, W)] \\
        &= \mathbb{P}_W[W \in \mathcal{C}(z^n)],
    \end{align}
    so $H(Y^n|X^n) = -\mathbb{E}_{Z^n} \mathbb{E}_{W}[\log p(Y^n|X^n)] = -\mathbb{E}_{Z^n} [ \log (\mathbb{P}_W[W \in \mathcal{C}(Z^n)]) ]$ as claimed.
\end{proof}
\renewcommand*{\proofname}{Proof.}

\newpage
\section{Bounds on smooth model families}\label{app:smooth.models}
While the focus of this paper is on realizable models, we note that these techniques are also applicable to smooth model families, and we can use them to recover classic results. 

\subsection{$R(D)$ for smooth model families}
We consider first the simpler case of log loss over model families that satisfy a smoothness condition. In this setting, we employ a quadratic approximation to the KL divergence, and $R(D)$ is approximated by the rate-distortion function of a Gaussian source.

First, we give a technical definition of the required smoothness conditions.
\begin{definition}[Smooth model family]
    We say a model family $P_{Z|W}$ is {\em smooth} if the following conditions hold: (1) The Fisher information matrices $\mathcal{I}_{y|x,W}$ and $\mathcal{I}_{z|W}$ exist, are nonsingular, and have finite expectation both elementwise and in terms of their determinants, and (2) the maximum-likelihood estimator of $W$ from $Z^n$ achieves the Cramer-Rao bound as $n \to \infty$; i.e. the estimation error converges on a Gaussian with inverse covariance given by the Fisher information matrix.
\end{definition}
These conditions are taken from \citet{ClarkeB90}, who show that these conditions imply a tractable approximation on $I(Z^n;W)$ in terms of the Fisher information. Under these conditions, we can bound $R(D)$ with respect to the log loss.

\begin{theorem}\label{thm:quadratic.rd}
    Under log loss and a smooth model family $P_{Z|W}$, the rate-distortion function is bounded by
    \begin{equation}
    R(D) \geq \Bigg[ h(W) - \frac{d_w}{2}\log\left(\frac{4\pi e D}{d_w}\right) +
    \frac{1}{2}\log |\mathbb{E}[\mathcal{I}_{Y|X,W}]| + o_n(1) \Bigg]^+.
\end{equation}
\end{theorem}

This bound resembles the classic Gaussian rate-distortion function, which is due to the quadratic approximation of the KL divergence. In particular, Theorem~\ref{thm:quadratic.rd} holds exactly (without the $o_n(1)$ term) if the model class is Gaussian.

\subsection{$I(Z^n; W)$ for smooth model families}\label{sect:smooth.mi.results}
As mentioned above, we need to estimate both $R(D)$ and $I(Z^n;W)$ to bound the Bayes risk, and a classic result by \citet{ClarkeB90} gives a tractable asymptotic expression for $I(Z^n;W)$.

\begin{lemma}[\citet{ClarkeB90}]\label{lem:data.mi}
For smooth model family $P_{ZW}$, the following approximation holds as $n \to \infty$:
\begin{equation}\label{eqn:data.mi}
	I(Z^n; W) = \frac{d_w}{2} \log \left( \frac{n}{2\pi e} \right) +  \frac{1}{2}\mathbb{E}_W[\log |\mathcal{I}_{Z|W}|]  +
	h(W)  + o_n(1),
\end{equation}
where $\mathcal{I}_{Z|w}$ is the Fisher information matrix evaluated at $w$, i.e. the Hessian of $\mathbb{E}_Z [-\log p(Z|w)]$. Note that this is different from the Fisher information matrix used in Lemma \ref{thm:quadratic.rd}.
\end{lemma}

Applying Theorem \ref{thm:dpi} to the above bound on $I(Z^n;W)$ and the  bound on $R(D)$ from Theorem \ref{thm:quadratic.rd} gives the following bound on the excess risk.
\begin{corollary}
    Under log loss and a smooth model family $P_{Z|W}$, the excess risk for any learning rule $Q_{\widehat{W}|Z^n}$ is bounded by
    \begin{equation}
        E_n \geq \frac{d}{2n} \cdot \frac{\left|\mathbb{E}[\mathcal{I}_{Y|X,W}]\right|^{1/d}}{ \left|\mathbb{E}[\mathcal{I}_{Z|W}]\right|^{1/d}} (1 + o_n(1)).
    \end{equation}
\end{corollary}

The KL divergence scales as $\Omega(d/n)$ as expected, and the scale factor is proportional to the {\em geometric mean} of the eigenvalues of the matrix $\mathbb{E}[\mathcal{I}_{Y|X,W}]\mathbb{E}^{-1}[\mathcal{I}_{Z|W}]$. Very roughly, this quantity expresses the ratio of how much information is needed to describe a test sample $Y$ given $X$ and $W$ and how much information a labeled sample conveys about $W$.

One can also use Lemma \ref{lem:data.mi} to derive excess risk on the 0-1 loss of smooth model families that are also $\mu$-sensitive in expectation. Doing so yields $\Omega(\sqrt{d/n})$ bounds that scale similarly to standard outer and minimax bounds on generalization error, as well as the scaling laws on the minimum excess risk proven by \citet[Theorem 6]{XuR20}.

\end{document}